\documentclass[final,12pt]{colt2026} 

\usepackage{bbm}
\usepackage{cleveref}
\usepackage{thm-restate}

\newtheorem{thm}{Theorem}
\newtheorem{lem}{Lemma}

\newtheorem{prop}[thm]{Proposition}

\newtheorem{ass}[theorem]{Assumption}

\usepackage{booktabs}      
\usepackage[table]{xcolor} 
\usepackage{tabularray}    
\usepackage{pifont}        

\DeclareMathOperator*{\argmax}{arg\,max}

\newcommand{\E}{\mathop{\mathbb{E}}}
\newcommand{\Prob}{\mathbb{P}}

\newcommand{\As}{\mathcal A}

\newcommand{\Fs}{\mathcal F}

\newcommand{\Os}{\mathcal O}

\newcommand{\Rs}{\mathcal R}

\newcommand{\Us}{\mathcal U}

\title[Online Market Making]{Online Market Making and the Value of Observing the Order Book}
\usepackage{times}

\coltauthor{%
 \Name{Davide Maran} \Email{davide.maran@polimi.it}\\
 \addr Politecnico di Milano, Piazza Leonardo da Vinci, 32-36 - Città Studi, Milano (MI)\\
 \AND
 \Name{Marcello Restelli} \Email{marcello.restelli@polimi.it}\\
 \addr Politecnico di Milano, Piazza Leonardo da Vinci, 32-36 - Città Studi, Milano (MI)%
}

\begin{document}

\maketitle

\begin{abstract}
We study an online market-making problem in which a learner sequentially posts bid and ask prices for a single asset while interacting with traders holding private valuations. Unlike existing online learning formulations that assume fully censored feedback, we introduce an action-dependent feedback model inspired by real limit order books: when a trade occurs, the trader’s valuation remains hidden, whereas when no trade occurs, informative feedback about supply and demand is revealed. 
We show that this additional information fundamentally changes the learnability of the problem. In the stochastic setting with i.i.d.\ market prices, we propose an elimination-based algorithm that achieves $\widetilde \Os(\sqrt{T})$ regret with high probability, without requiring any smoothness assumptions on the distribution of trader valuations. We then extend this result to a broad class of mean-reverting price processes by considering both local, autoregressive dynamics and a weaker global drift condition based on cumulative deviations from the mean. Under either assumption, we establish high-probability $\widetilde \Os(\sqrt{T})$ regret bounds, relying on a new concentration inequality of independent interest. Finally, in the adversarial setting with oblivious prices, we design an explore-then-perturb algorithm that guarantees $\tilde \Os(T^{2/3})$ regret in expectation.
Our results quantify the value of observing the order book in online market making and demonstrate that even limited, action-dependent feedback can substantially improve regret guarantees compared to standard bandit feedback models.
\end{abstract}


\section{Introduction}

Market making is the activity performed by intermediaries who provide liquidity to an asset by simultaneously quoting buying prices (bid) and selling prices (ask). This function is essential for the efficiency of financial markets, as it reduces transaction costs and facilitates the immediate matching of supply and demand \citep{amihud1986asset}. Without market makers, investors would face wider spreads and increased price volatility \citep{glosten1988estimating, madhavan2000market}.

In the context of online learning, market making is framed as an iterative game between the agent, "maker", and a taker (which models the rest of the market). The agent decides sequentially one bid/ask pair $B_t,A_t$, without knowing the taker's private valuation $V_t$ or the future market value of the asset $M_t$, aiming to minimize regret relative to the best fixed strategy in hindsight. This perspective introduces a fundamental exploration-exploitation trade-off, where the agent must balance learning the latent distribution of trader valuations with the goal of maximizing immediate profit. Previous work in this domain \citep{cesa2024market} focused on a restricted feedback model where the agent only observes the market value $M_t$ and a binary indicator of whether a transaction occurred (i.e., $V_t \le B_t$ or $V_t \ge A_t$). Under such limited information, only weak regret guarantees, typically sublinear only under restrictive assumptions, can be achieved. 

This feedback accurately reflects the fact that, whenever the agent makes a deal, it is not possible to understand which was the true valuation $V_t$ of the taker. Nonetheless, we observe that, in the opposite case, when no deal happens, it is fair to assume the $V_t$ is revealed to the learner. In fact, this assumption aligns with the operational reality of modern electronic exchanges, where the Limit Order Book (LOB) serves as a public ledger of intent. While a completed transaction only reveals that a price was met, the absence of a trade at the maker's spread allows the agent to observe the surrounding "resting" limit orders, which explicitly represent the buy and sell valuations of other market participants. Observing the book of limit orders is vital for every modern market maker. This information is so valuable that industry leaders such as Jane Street, Citadel Securities, and Hudson River Trading invest hundreds of millions of dollars annually in high-speed, "Level 3" market data feeds to gain full visibility into these unexecuted orders. Global spending on financial market data reached a record \$44 billion in 2024 \citep{burton2025marketdata}, reflecting the industry’s consensus that observing the order book is not a luxury, but a fundamental requirement for effective market making.

\subsection{Original contribution}

The main contribution of this paper is the introduction of a novel action-dependent feedback model for online market making that closely reflects the information structure of real limit order books.
Unlike classical formulations, where feedback is either fully censored (bandit feedback) or fully revealed (full feedback), we propose a feedback mechanism in which the information received by the learner depends continuously on the chosen bid–ask spread.
This feedback structure captures a fundamental feature of real-world market making: information about demand and supply is revealed precisely when liquidity is \textit{not} consumed.
From an Online Learning perspective, this setting gives rise to a previously unexplored regime in which:
(i) the feedback is partial and censored,
(ii) the type of feedback depends on the learner’s action,
and (iii) informative feedback is obtained exactly when the instantaneous reward is zero.
To the best of our knowledge, this is the first formalization of such a feedback structure in the field of Online Learning.
Building on this model, we establish regret guarantees in three progressively more challenging environments:
\begin{enumerate}
    \item Stochastic prices. When the market price process is i.i.d. with unknown mean, we design an algorithm that achieves $\widetilde \Os(\sqrt T)$ regret with high probability.
    \item Mean-reverting prices.
    We extend the stochastic analysis to a broad class of mean-reverting price processes, significantly relaxing the i.i.d. assumption. Under a mild martingale-type condition, we show that the same $\widetilde \Os(\sqrt T)$ regret rate can still be achieved. This result is based on a simple yet novel concentration inequality of independent interest.
    \item Adversarial prices. When market prices are allowed to be an arbitrary oblivious sequence, we propose an explore-then-perturb algorithm that guarantees $\widetilde \Os(T^{\frac{2}{3}})$ regret in expectation. 
\end{enumerate}
In contrast, lower bounds for the setting with bandit feedback \citep{cesa2024market} show that a regret of $\Omega(T^{\frac{2}{3}})$ in the stochastic case, and of $\Omega(T)$ in the adversarial unless a Lipschitz condition on the c.d.f. of $V_t$ is met.
Our results quantify the value of observing the order book in online market making and show that even a limited, action-dependent form of feedback can dramatically improve theoretical guarantees.

\section{Setting}

In this paper, we model the market-making problem as a discrete-time online learning interaction between an agent (the "maker") and a sequence of "takers". Formally, the Online Market Making setting is defined by two main components: first, a sequence of takers who, at each round $t=1, \dots, T$, hold a private valuation $V_t$ for a single unit of the underlying asset; second, a sequence of market prices $M_t$ representing the asset's objective value at the end of each round.
At each round, the market maker selects a pair of bid and ask prices $(B_t, A_t)$ from the available action space, without prior knowledge of either $V_t$ or $M_t$. Consequently, three scenarios may arise:
\begin{itemize}
    \item \textbf{Maker Buys:} If $V_t \le B_t$, the taker sells the asset to the maker at price $B_t$, resulting in a reward of $M_t - B_t$ for the agent.
    \item \textbf{No Trade:} If $B_t < V_t < A_t$, no transaction occurs, and the agent receives a reward of $0$.
    \item \textbf{Maker Sells:} If $V_t \ge A_t$, the taker purchases the asset from the maker at price $A_t$, yielding a reward of $A_t - M_t$ for the agent.
\end{itemize}
Importantly, in real markets, the agent does not interact with a single taker at a time, and no taker is willing to buy or sell at the same $V_t$.
Focusing on a single private valuation $V_t$ per round is a necessary simplification, somewhat close to the Glosten-Milgrom model \citep{das2005learning, touzo2021information}, which serves as a powerful abstraction that captures the idea of "best opportunity of the market". In mathematical terms, the reward function of the agent writes as
\begin{equation}
    r(b,a; v,m):=\mathbbm{1}(v \le b)(m-b)+\mathbbm{1}(a \le v)(a-m)\label{eq:rewfun}
\end{equation}
for $v=V_t$ and $m=M_t$. The following assumptions, introduced by \cite{cesa2024market}, will be always considered:
\begin{ass}[Bounded price]\label{ass:bounded}
    At any time step $V_t,M_t\in [0,1]$.
\end{ass}
\begin{ass}[Stochastic evaluation]\label{ass:stoch_val}
    The evaluations $V_t$ are sampled from a probability distribution with c.d.f. $F:[0,1]\to [0,1]$ independently at any time step. By convenience, we always call $S(x):=1-F(x)$ the corresponding survival function.
\end{ass}
From \Cref{ass:bounded}, it follows that the agent has no need to choose any pair $(b,a)$ outside of the following set $\Us:=\{b,a: 0\le b < a\le 1\}.$
The former thus defines the agent's action space. We do not yet need explicit assumptions on $M_t$. Different results apply when $M_t$ is stochastic, i.i.d., and independent of $V_t$ (stochastic case) versus when it is an arbitrary sequence, as shown by \cite{cesa2024market}. Our main novelty relative to that paper is the following "order book access" assumption.
\begin{ass}[Order book access]\label{ass:order}
    At the end of the round, if $\boxed{B_t< V_t< A_t}$ the agent observes $V_t,M_t$, otherwise only $M_t$.
\end{ass}
In other words, \textit{information arrives exactly when reward does not}. The former assumption reflects a real problem of market makers: if $V_t$ falls outside of $(B_t, A_t)$, the deal occurs at one of the extrema, and the agent has no possibility to infer the value at which the taker was willing to buy or sell; if $V_t$ falls within the spread, its value becomes observable through the clearing of limit orders, allowing the agent to capture the transaction price resulting from the interaction between other market participants. 
From an Online Learning perspective, this feedback structure is particularly noteworthy. Specifically, it yields full feedback when $V_t \in (B_t, A_t)$ and bandit feedback otherwise. Given that the agent receives zero reward whenever $V_t$ falls within this interval, such a hybrid feedback mechanism significantly complicates the exploration-exploitation trade-off. 
The agent's goal is to minimize regret, measured w.r.t. the optimal bid-ask choice in hindsight.
Calling
\begin{equation}J_t(a,b):=F(b)(M_t-b)+S(a)(a-M_t),\label{eq:obj}\end{equation}
\Cref{ass:stoch_val} allows us to write the regret as follows
\begin{align}
    R_T&=\sup_{b,a\in [0,1]}\E_{V_t}\left[\sum_{t=1}^T r(b,a,V_t,M_t)-r(B_t,A_t,V_t,M_t)\right]\label{eq:regdef}\\
    &=\sup_{b,a\in [0,1]}\sum_{t=1}^T J_t(b,a)-J_t(B_t,A_t).
\end{align}
Given the particular structure of $J_t$ in \Cref{eq:obj}, the supremum of the sum which appears in \Cref{eq:regdef} may be written in closed form as
\begin{equation}
    \sup_{b,a\in [0,1]}\frac{1}{T}\sum_{t=1}^T J_t(b,a)=\sup_{b,a\in [0,1]}(\widehat \mu_T-b)F(b)+(a-\widehat   \mu_T)S(a),\qquad \widehat   \mu_T=\frac{1}{T}\sum_{t=1}^TM_t.\label{eq:obj2}
\end{equation}
As already noted, the former definition says nothing about the sequence $M_t$.
In the following section, we are going to explore the theoretical guarantees that correspond to some assumptions on the market price. The following sections consider different scenarios in order of generality. First, the case of a sequence of independent random variables (\cref{sec:stoc}), then a generalization which introduces different forms of temporal correlation which formalize the concept of \textit{mean reversion} (\cref{sec:gene}) and, finally, the one when $M_t$ is an arbitrary sequence which cannot adapt to the choices of the agent (\cref{sec:adv}).

\section{Stochastic Independent Market Prices}\label{sec:stoc}
We begin by examining the case where $M_t$ is an independent stochastic process.
In this section, we assume the market price follows an independent process with unknown mean $\mu$. 
\begin{ass}[Stochastic independent market price]\label{ass:stoch_mark}
    The evaluations $M_t$ are an independent sequence with mean $\mu$.
\end{ass}
Under 
\Cref{ass:stoch_mark}, the environment is fully stochastic. In this setting, the most natural approach is to estimate $\mu$ with a sample mean estimator $\widehat \mu_t$ and plan the bid-ask spreads in order to be a near-maximizer of a surrogate reward function which takes the form $\widetilde F(b)(\widehat \mu_t-b)+\widetilde S(a)(a-\widehat \mu_t)$. 

Dealing with the estimator for the mean is the easy part. Indeed, by a relatively standard version of Azuma-Hoeffding's inequality that we show in appendix \ref{app:stoc}, with probability at least $1-\delta$,
\begin{equation}
    \forall 1\le t\le T\qquad |\widehat\mu_t-\mu|\le \sqrt{\frac{\log(3\log(T)/\delta)}{t}}.\label{eq:hoef2}
\end{equation}
Since the sample mean converges quickly to the true mean, in this setting there is essentially no difference between measuring the regret \Cref{eq:obj2} with respect to $\mu$ or $\widehat \mu_T$. In similar stochastic settings, one therefore often uses the following notion of regret (pseudo-regret) as the objective.
\begin{equation}
    \Rs_T=T\cdot \sup_{b,a\in [0,1]}J(b,a)-\sum_{t=1}^T J(B_t,A_t)\qquad J(b,a)=F(b)(\mu-b)+S(a)(a-\mu).\label{eq:pseudo}
\end{equation}

In the stochastic setting, pseudo-regret is close to the regret with high probability, as formalized in \Cref{prop:stoch_reg}.
The other ingredient to perform this step is an estimator $\widetilde F$ for the c.d.f. of $V_t$. Define the empirical c.d.f. as follows:
\begin{equation}
    \widehat F_t(x)=\frac{\sum_{\tau=1}^t \mathbbm{1}(V_\tau\le x)}{t}.\label{eq:ecdf}
\end{equation}
By the Glivenko-Cantelli theorem (see Chapter 19 in \citep{van2000asymptotic}), this estimator is known to converge to $F$ \textit{uniformly}. Unfortunately, in our interaction $\widehat F_t$ cannot always be built, as the agent is only aware of $\text{clip}(V_t; B_t,A_t)$ at any step $t$, by \Cref{ass:order}. When playing a spread $A_t-B_t$ that is too narrow, the agent gathers no information to improve this estimate. On the other hand, when the agent receives information, i.e., $B_t< V_t< A_t$, no trade occurs; therefore, the agent receives no reward. To address these two problems, we will propose an algorithm that is both elimination-based, to ensure that the c.d.f. of all active prices is estimated at each time step, and optimistic, to ensure that the optimal pair $a,b$ is not discarded with high probability. 

\subsection{Algorithm}

The learner faces a tension between narrowing the spread to gain reward and widening it to gain information. OPSR (\cref{alg:OPSR}) resolves this by shrinking the action set only when confidence intervals certify suboptimality, while always playing the most conservative surviving quotes.
This algorithm keeps in memory the value of the e.c.d.f. for all "interesting" prices, which need to be in a discrete set.
Therefore, define the following uniform quantization
$$\As := \left\{ \frac{n}{\lceil \sqrt{T} \rceil} : n = 0, 1, \dots, \lceil \sqrt{T} \rceil \right\}.$$
At any time-step, our algorithm is going to choose a pair $B_t,A_t\in \As$ and we will call $\As_t:=\As \cup [B_t,A_t]$. 
While the agent cannot always access the value $V_t$, which is formally required to build \cref{eq:ecdf}, the observation is sufficient to build this surrogate 
$$V_t^\text{clip}:=
\begin{cases}
B_t&\quad V_t< B_t\\
V_t&\quad B_t\le V_t\le A_t\\
A_t+0.01&\quad V_t>A_t
\end{cases}\qquad \chi_t(x):=\mathbbm{1}(V_t^\text{clip}\le x).$$
If $x\in [B_t,A_t]$, then $\{V_t^\text{clip}\le x\}$ corresponds to $\{V_t\le x\}$, whichever the value of $V_t$ (in fact, the choice of $0.01$ is arbitrary, any positive number would work for for purposes). Taking any price $x$ that belongs to the whole sequence of spreads, the following proposition holds.
\begin{restatable}{prop}{ecdf}\label{prop:ecdf}
    For any $t\in [T], x\in \bigcap_{\tau=1}^t \As_\tau$,
    $\widehat F_t(x)=\frac{1}{t}\sum_{\tau=1}^t\chi_\tau(x)$ and, for any $\delta>0$,
    $$\Prob\left(\exists t\in [T],\ \exists x \in \bigcap_{\tau=1}^t \As_\tau,\quad |\widehat F_t(x)-F(x)|>\sqrt{\frac{3\log(3T/\delta)}{4t}}\right)\le \delta.$$
\end{restatable}
The former proposition shows that the estimated c.d.f. concentrates nicely, almost as fast as $\widehat \mu_t$ for $x\in \bigcap_{\tau=1}^t \As_\tau$. Outside of this set, the estimation of $\widehat F_t$ gets more difficult, if not impossible. Therefore, the idea of our algorithm is to let the sets $\As_t$ be monotonically decreasing, so that $\bigcap_{\tau=1}^t \As_\tau=\As_t$. In this way, the agent is guaranteed to have all the information available for all "interesting" prices at any time step. Once $x\notin \As_t$ at one time step, meaning $x<B_t$ or $x>A_t$, there is no chance this price is played again. This idea is the same as elimination-based algorithms \citep{even2006action}. Eliminating an arm is risky and may lead the agent to pay linear regret. Therefore, we define the following optimistic/pessimistic estimators for the c.d.f.
\begin{align}
    \widehat F_t^\text{up}(x)=\min\{\widehat F_{t-1}^\text{up}(x),\widehat F_t(x)+&\psi(t,\delta)\},\qquad \widehat F_t^\text{low}(x)=\max\{\widehat F_{t-1}^\text{low}(x),\widehat F_t(x)-\psi(t,\delta)\}\label{eq:bid_ci}\\
    &\psi(t,\delta):=\sqrt{\frac{3\log(3T/\delta)}{4t}},
\end{align}
where $\widehat F_0^\text{up}(x):=1$ and $\widehat F_0^\text{low}(x):=0$.
When dealing with asks instead of bids, it is more natural to talk about the survival function than the cumulative distribution. Indeed, to get its optimistic/pessimistic estimators, one just needs to define $\widehat S_t^\text{up}(x)=1-\widehat F_t^\text{low}(x)$ and $\widehat S_t^\text{low}(x)=1-\widehat F_t^\text{up}(x)$. The same can be done for $\mu$, exploiting \cref{eq:hoef2},
\begin{align}
    \widehat \mu_t^\text{up}:=\min\{\widehat \mu_{t-1}^\text{up},\widehat \mu_t &+\phi(t,\delta)\} \qquad \widehat \mu_t^\text{low}:=\max\{\widehat \mu_{t-1}^\text{low},\widehat \mu_t -\phi(t,\delta)\}\\ &\phi(t,\delta):=\sqrt{\frac{\log(3\log(T)/\delta)}{t}}\label{eq:mu_ci},
\end{align}
where $\widehat \mu_0^\text{up}:=1, \widehat \mu_0^\text{low}:=0$.
The algorithm presented in \cref{alg:OPSR} leverages all available estimators to eliminate sub-optimal bid-ask pairs until it converges to a quasi-optimal solution.
Its scheme alternates between optimistic and pessimistic estimates for each element in $\As_{t-1}$. At the beginning of each round, optimistic estimates for $(x-\mu)F(x)$ and $(\mu-x)S(x)$ are computed in lines \ref{algline:optibid} and \ref{algline:optiask}, respectively. Subsequently, pessimistic estimates are calculated for every element in $\As_{t-1}$ (line \ref{algline:pessbid} and line \ref{algline:pessask}), from which the respective maxima $\Gamma_\text{bid}^*$ and $\Gamma_\text{ask}^*$ are derived. Finally, the spread is narrowed in lines \ref{algline:restbid} and \ref{algline:restask} by increasing $B_t$ until an element satisfying $\Theta_\text{bid}(x) \ge \Gamma_\text{bid}^*$ is encountered, with $A_t$ being decreased analogously. This last step is crucial: even if heavily sub-optimal bid-ask pairs remain within $\As_t$, they do not increase the regret, as the played values are always the extreme points of the interval.

\begin{algorithm2e}[t]
\everypar={\nl}
\RestyleAlgo{ruled}
\LinesNumbered
\caption{Optimistic/Pessimistic Successive Rejects (OPSR)}\label{alg:OPSR}
\DontPrintSemicolon
\KwData{Time horizon $T$.}
Set $B_0\gets 0,A_0\gets 1$\label{algline:init}\;

\For{$t = 1,\dots T$}{

    Play $(B_{t-1}, A_{t-1})$\label{algline:play}\;

    Update \Cref{eq:bid_ci} and \Cref{eq:mu_ci}\;

    \For{$x \in \As_{t-1}$}{

        $\Theta_\text{bid}(x)\gets (\widehat \mu_t^\text{low}-x)\widehat F_t^\text{up}(x)$\label{algline:optibid}\;
    
        $\Theta_\text{ask}(x)\gets (x-\widehat \mu_t^\text{up})\widehat S_t^\text{up}(x)$\label{algline:optiask}\;

        $\Gamma_\text{bid}(x)\gets (\widehat \mu_t^\text{low}-x)\widehat F_t^\text{low}(x)$\label{algline:pessbid}

        $\Gamma_\text{ask}(x)\gets (x-\widehat \mu_t^\text{up})\widehat S_t^\text{low}(x)$\label{algline:pessask}
    }
    
    $\Gamma_\text{bid}^* \gets \max_{x\in \As_{t-1}}\Gamma_\text{bid}(x)$\;
    $\Gamma_\text{ask}^* \gets \max_{x\in \As_{t-1}}\Gamma_\text{ask}(x)$\;

    $B_t\gets \min\{a\in \As_{t-1}: \Theta_\text{bid}(a)\ge \Gamma_\text{bid}^*\}$\label{algline:restbid}\;

    $A_t\gets \max\{a\in \As_{t-1}: \Theta_\text{ask}(a)\ge \Gamma_\text{ask}^*\}$\label{algline:restask}\;
}
\end{algorithm2e}

This structure allows OPSR to have a regret guarantee under our assumptions. First, we show that these guarantees hold outside the following failure event. Define
\begin{equation}E:= \bigcup_{t=1}^T\left\{\bigcup_{a\in \As_t} |\widehat F_t(x) - F(x)|>\psi(t,\delta)\right\}\cup \left\{|\widehat \mu_t-\mu|> \phi(t,\delta)\right\}\label{eq:failure}.\end{equation}
This event corresponds to the fact that at least one of our confidence bounds in \Cref{eq:bid_ci} and \Cref{eq:mu_ci} fails. As already proved in \Cref{eq:hoef2} and \Cref{prop:ecdf}, $\Prob(E)$ does not exceed $2\delta$. The following theorem shows that, as long as $E$ does not hold, the regret is under control.
\begin{restatable}{thm}{regretstoch}\label{thm:regstoc}
    Under Assumptions \ref{ass:bounded} and \ref{ass:order}. Under $E^c$ (the failure event in \Cref{eq:failure}), the pseudo-regret suffered by \Cref{alg:OPSR} is bounded by
    $\Rs_T\le\sqrt T+2 + 4\sum_{t=1}^{T-1} \psi(t,\delta)+\phi(t,\delta)$
    which, for $\phi,\psi$ as in Equations \eqref{eq:mu_ci} and \eqref{eq:bid_ci}, writes as
    $$\Rs_T\le \sqrt{48T\log(3T/\delta)}+\Os\left(\sqrt{T\log (\log (T))}\right).$$
\end{restatable}
Building on the previous result, we derive an upper bound on the regret by bounding the failure probability and the gap between the regret and the pseudo-regret. By \Cref{thm:regstoc}, \Cref{prop:ecdf}, and \Cref{prop:stoch_reg}, for \Cref{alg:OPSR} we have, with probability at least $1-\delta$,
$$R_T\le \sqrt{48T\log(9T/\delta)}+\Os\left(\sqrt{T\log (\log (T))}\right).$$

This result mirrors several bounds for stochastic bandits \cite{lattimore2020bandit}. Interestingly, while continuous bandits with an action space in the interval $[0,1]^2$ suffer a regret of $\widetilde{\mathcal{O}}(T^{3/4})$ under Lipschitz reward assumptions \cite{kleinberg2004nearly}, or $\widetilde{\mathcal{O}}(\sqrt T)$ under much stronger assumptions \cite{bubeck2011x,liu2021smooth}, our algorithm achieves $\widetilde{\mathcal{O}}(\sqrt{T})$ regret without requiring any smoothness assumptions on the cumulative distribution function.

\paragraph{Computational complexity}
In the domain of market making, latency is a critical bottleneck for any algorithmic framework. The high-frequency nature of modern markets necessitates near-instantaneous decision-making, as liquidity opportunities often vanish within a few milliseconds. Consequently, there is significant interest in designing algorithms that simultaneously provide robust no-regret guarantees and maintain low per-step computational complexity. Our proposed \cref{alg:OPSR} involves updating the estimates $\widehat F_t(x)$ for all $x \in \As_t$, followed by extracting the maxima $\Gamma_{\text{bid}}^*$ and $\Gamma_{\text{ask}}^*$. These operations entail a per-step complexity of $\Os(|\As|) = \Os(\sqrt{T})$. Notably, despite the bivariate nature of the optimization problem, we avoid $\Os(|\As|^2)$ complexity due to the additive separability of the reward function with respect to bid and ask prices.

To improve computational efficiency without affecting the regret bound, we introduce the \textsc{lazyOPSR} algorithm (\Cref{alg:lazyOPSR}, detailed in the Appendix). This variant stores $V_t$ and updates $\widehat F_t$ only at time steps $t$ that are powers of two. For most iterations, the per-step complexity is $\Os(1)$. In the remaining $\Os(\log T)$ iterations, $\widehat F_t$ is reconstructed from all past data in $\Os(T + |\As|) = \Os(T)$ using prefix sums. Thus, the amortized per-step complexity is $\Os(\log T)$.

\section{Stochastic Prices with Temporal Correlation}\label{sec:gene}

In the former section, we covered the stochastic independent case. The proof of the regret bound relies on the definition of a failure event and ensures small regret whenever that event is not verified. By its definition (see \Cref{eq:failure}), whenever one can ensure that the event
\begin{equation}\bigcup_{t=1}^T\left\{|\widehat \mu_t-\mu|> \phi(t,\delta)\right\}\qquad \phi(t,\delta)\approx\sqrt{\frac{\log(t/\delta)}{t}}\label{eq:discrep}\end{equation}
does not happen, \Cref{thm:regstoc} ensures small (pseudo) regret. Assuming that $M_t$ forms an independent sequence is a good way to say that the event in \Cref{eq:discrep} has low probability, but there do exist many other processes with this quality. For example, let
\begin{equation}
    M_1=\text{Unif}(0,1),\qquad M_{t+1}=1-M_t.
    \label{eq:examr}
\end{equation}
This sequence is not independent. Nonetheless, their sample mean $\widehat \mu_t$ is always either $1/2$, for $t=2n$ or $\frac{M_1+n}{2n+1}$, for $t=2n+1$. In both cases, $|\widehat \mu_t-\mu|\le 1/t$, a rate that is even faster than the i.i.d. case.
Stochastic processes of this nature, characterized by a negative correlation between the current value and its preceding deviations from the long-term equilibrium, informally defined as mean-reverting processes. These models are of paramount importance in the field of quantitative finance, as a vast array of algorithmic trading strategies is predicated on mean-reverting assumptions about asset price dynamics. 
We formalize this phenomenon with two assumptions: one based on the Ornstein-Uhlenbeck process, capturing short-term mean reversion, and another with a martingale viewpoint, capturing mean reversion relative to the process’s entire history.
\begin{ass}[Local mean reversion]\label{ass:ou_mar}
    For some arbitrary initial conditions $\{M_t\}_{t=-\tau}^0\subset [0,1]$, the market price evolves according to the following $AR(k)$ process:
    $$\forall t\in \mathbb N\qquad M_{t+1}=\sum_{\tau=0}^{k-1}\gamma_\tau M_{t-\tau}+(1-\gamma)\eta_{t+1},$$
    where $\eta_{t+1}$ is a random variable with support in $[0,1]$ and such that $\E[\eta_{t+1}|\Fs_t]=\mu$. $\{\gamma_\tau\}_{\tau=0}^{k-1}$ are non-negative real numbers such that $\sum_{\tau=0}^{k-1}\gamma_\tau=\gamma <1$.
\end{ass}
For $k=1$, the former assumption covers the discrete Ornstein-Uhlenbeck process \citep{grimmett2020probability}, which is often used to model mean reversion. In that case, \cref{ass:ou_mar} may be written as $\Delta M_{t+1}=(1-\gamma)(\eta_{t+1} -M_t)$, with $\E[\eta_{t+1}|\Fs_t]=\mu$.

\begin{ass}[Global mean reversion]\label{ass:mr_mar}
    There is some constant $\mu$ such that, given $X_t:=M_t-\mu$ and $S_t=\sum_{\tau=1}^tX_\tau,$ one has, for every $t\le T$
    $$\E[X_{t+1}|\Fs_t]\cdot S_t\le 0.$$
\end{ass}
Assumptions \ref{ass:ou_mar} and \ref{ass:mr_mar} represent two possible formalizations of the mean-reverting mechanism; indeed, analogous definitions have been established in the literature (see, e.g., \citep{vasicek1977equilibrium, bouchaud2003theory}). 
While the scope of the two assumptions differs slightly, they both capture interesting scenarios.
If the market prices are independent as in the previous section, \Cref{ass:ou_mar} holds for $k=1,\gamma=0$ and \Cref{ass:mr_mar} is satisfied with $\mu=\E[M_1]$ as $\E[X_{t+1}|\Fs_t]=0$. On the other side, only \Cref{ass:mr_mar} covers the sequence of random variables in \Cref{eq:examr}.
In both cases, we can prove concentration inequalities that allow to generalize the independent case.
\begin{restatable}{thm}{martingaleMR}\label{thm:martingaleMR}
    Under Assumption \ref{ass:bounded} and one between Assumptions \ref{ass:ou_mar} and \ref{ass:mr_mar}.\\ The event $\bigcup_{t=1}^T\left\{|\widehat \mu_t-\mu|> \overline\phi(t,\delta)\right\}$ has a probability not larger then $\delta$, where the specific form of $\overline \phi$ depends on the assumption.
    \begin{enumerate}
        \item Under \Cref{ass:ou_mar},
        $\overline\phi(t,\delta):= \sqrt{\frac{\log(2T/\delta)}{4t}}+\frac{\gamma (k+1)}{(1-\gamma)t}.$
        \item Under \Cref{ass:mr_mar},        $\overline\phi(t,\delta):=\sqrt{\frac{4\log(2T/\delta)}{t}}.$
    \end{enumerate}
\end{restatable}
In the second case, the proof follows by a novel martingale argument of independent interest, \Cref{thm:newmart}. To our knowledge, this is the first concentration result that exploits global mean-reversion expressed through cumulative deviations rather than mixing or spectral assumptions. The formed value of $\overline \phi$ is worse than $\phi$ that we used in \cref{sec:stoc}, by only logarithmic factors ($\log \log T$ becomes $\log T$).
After \Cref{thm:martingaleMR}, nothing more needs to be proved to show the bound on the pseudo-regret. Indeed, 
\Cref{thm:regstoc}, which does not assume any specific form for $\phi$, can be applied for the two values of $\phi\gets \overline \phi$ in \Cref{thm:martingaleMR} giving, under $E^c$,
\begin{align*}
    \Rs_T&\le\sqrt T+2 + 4\sum_{t=1}^{T-1} \psi(t,\delta)+\overline\phi(t,\delta)\\
    &\le \begin{cases}
        \sqrt{36\cdot T\log(3T/\delta)}+\frac{\gamma (k+1)\log(T)}{(1-\gamma)}+\Os(\sqrt T)&\text{under \Cref{ass:ou_mar}}\\
        &\\
        \sqrt{526\cdot T\log(3T/\delta)}+\Os(\sqrt T)\qquad &\text{under \Cref{ass:mr_mar}}
    \end{cases}
\end{align*}
At this point, \Cref{thm:martingaleMR} shows that $\Prob(E)<2\delta$, so the former result is a valid high-probability pseudo-regret bound. Interestingly, with respect to the analogous result in the stochastic case, there is only some constant change, not the regret order, even when comparing the logarithmic terms.

To pass from the former result to an upper bound for the regret $R_T$ is not trivial. In fact, \Cref{prop:stoch_reg}, which bounds the discrepancy between $\Rs_T$ and $R_T$ is only valid in the stochastic setting. 
For this reason, we have to slightly modify \Cref{alg:OPSR}, by performing action eliminations only at times-steps $t$ such that $t=2^p$ for some $p\in \mathbb N$. After employing this technique, sometimes called the \textit{doubling trick} \citep{besson2018doubling}, we refer to the algorithm as \textsc{LazyOPSR}. For a detailed implementation of this algorithm, see the Appendix \ref{app:lazy_algo}.

\begin{restatable}{thm}{regretmr}\label{thm:regretmr}
    Under Assumptions \ref{ass:bounded}, \ref{ass:stoch_val}, \ref{ass:order} and one between \ref{ass:ou_mar} and \ref{ass:mr_mar}, the regret suffered by \textsc{LazyOPSR} is bounded, with probability at least $1-\delta$, by
    $$R_T\le \begin{cases}
        \sqrt{2007\cdot T\log(6T/\delta)}+\frac{8\gamma (k+1)\log(T)}{(1-\gamma)}+\Os(\sqrt T)&\text{under \Cref{ass:ou_mar}}\\
        &\\
        \sqrt{20783\cdot T\log(6T/\delta)}+\Os(\sqrt T)\qquad &\text{under \Cref{ass:mr_mar}}
    \end{cases}$$
\end{restatable}
The motivation for employing a scheme such as \textsc{LazyOPSR} in this setting stems from the partially adversarial nature of the environment. Given that \Cref{ass:mr_mar} is considerably weaker than the standard i.i.d. assumption on the sequence $\{M_t\}_{t=1}^T$, it is prudent to restrict the algorithm to a limited number of action switches, specifically $\Os(\log T)$ in our case. This "lazy" update schedule serves as a regularization mechanism, preventing the agent's choice from correlating with $M_t$ in an unpredictable way. Despite this doubling mechanism, which slows down learning, the regret is only a constant worse than the stochastic case (note $\sqrt{20783}\approx 144$ while $\sqrt{48}\approx 7$).

\section{Adversarial Market Values}\label{sec:adv}

After discussing the stochastic case and its generalization, in this section, we focus on the case where no assumption is put on $M_t$, which is allowed to be any sequence of values in $[0,1]$. This regime introduces significantly greater challenges: under such agnostic conditions, the historical realization of the sequence $\{M_t\}_{t \geq 1}$ provides no predictive information regarding its future evolution. Consequently, the empirical mean price $\widehat \mu_t$, which exhibited relative stationarity in the stochastic setting, may now undergo substantial fluctuations throughout the learning process. This volatility necessitates continuous exploration across the entire bid-ask action space, precluding any strategy aimed at a monotonic narrowing of the spread.

\begin{algorithm2e}[t]
\everypar={\nl}
\RestyleAlgo{ruled}
\LinesNumbered
\caption{Explore Then Perturb (ETP)}\label{alg:ETP}
\DontPrintSemicolon
\KwData{Time horizon $T$, number of exploratory rounds $\kappa$.}
$B_0\gets0,A_0\gets1$\\
\For{$t = 1,\dots \kappa$}{
   Play $(B_{t-1}, A_{t-1})$
   $B_t\gets0,A_t\gets1$
}
Compute $\widehat F$ with \Cref{eq:ecdf}

$\widehat S(\cdot)\gets 1-\widehat F(\cdot)$\label{algline:endetc}

$\widehat \mu_t\gets0$

Sample $\varepsilon\sim\text{Unif}(0,T^{-1/2})$\label{algline:noise}

\For{$t = \kappa+1,\dots T$}{
    Play $(B_{t-1}, A_{t-1})$

    $\widehat \mu_t\gets\frac{(t-\kappa-1)\widehat \mu_{t-1}+M_t}{t-\kappa}$

    $B_t\gets \text{argmax}_{b\in \As} \widehat F(b)(\widehat \mu_t+\varepsilon-b)$

    $A_t\gets \text{argmax}_{a\in \As} \widehat S(a)( a+\varepsilon-\widehat \mu_t)$
}
\end{algorithm2e}

Our strategy, \Cref{alg:ETP}, is to split the two problems, the first being the partial feedback, which allows to estimate $\widehat F(x)$ only for $x\in [B_t,A_t]$, and the other being that $\widehat \mu_t$ is not guaranteed to converge to any specific value. To begin with, \Cref{alg:ETP} utilizes the first $\kappa$ iterations with the environment to estimate the function $\widehat F(x)$ across the entire interval, maintaining $B_t=0$ and $A_t=1$ during this phase. Selecting bid and ask prices at the boundaries of the interval causes the regret to grow rapidly; however, it is the only way to ensure a uniformly good estimate of $F$. This initial phase, which continues until line \ref{algline:endetc}, is essentially based on the principles of the \textsc{Explore-then-Commit} algorithm \citep{garivier2016explore}.
The subsequent part of the algorithm no longer attempts to improve the estimate $\widehat F$; instead, it focuses entirely on minimizing the regret. The framework we draw inspiration from is the \textsc{Follow-the-Perturbed-Leader} algorithm \citep{kalai2005efficient}, the core idea of which is to inject random noise to perturb the cumulative reward function, thereby stabilizing the optimizer's decisions. This is implemented in line \cref{algline:noise}, where the noise $\varepsilon$ is sampled from the distribution $\text{Unif}(0,T^{-1/2})$. The utility of this seemingly counterintuitive choice lies in the need to prevent the sequence of chosen $\{B_t, A_t\}$ from being excessively volatile. In non-convex contexts such as ours, a minimal variation in the cumulative reward function (which, in our case, is essentially $\widehat F(b)(\widehat \mu_t-b)+\widehat S(a)( a-\widehat \mu_t)$) can cause abrupt shifts in its maximizer. The inclusion of noise allows us to prove that, on average, this instability remains bounded for any sequence of market prices $\{M_t\}_{t}$ that cannot adapt to $\varepsilon$.

Interestingly, the noise applied here differs from the usual implementations of FTPL. Here, instead of adding the noise to $\sum_{\tau=1}^t r_\tau(b,a)$, where $r_\tau$ is the reward relative to any pair $b,a$ at step $t$, the noise is added to $\widehat \mu_t$. This choice is suited to the particular structure of our problem, and we do not know if a more standard implementation of FTPL could achieve a similar regret bound.
\begin{restatable}{thm}{regretadv}\label{thm:regradv}
    Under Assumptions \ref{ass:bounded}, \ref{ass:stoch_val} and \ref{ass:order} and let $\{M_t\}_{t=1}^T$ be any unkown oblivious sequence. For $\kappa=\left \lceil \log(3T)^{\frac{1}{3}}T^{\frac{2}{3}}\right \rceil$, the regret suffered by \Cref{alg:ETP} is bounded, in expectation, by
    $$\E[R_T]\le 5\log(T)^{\frac{1}{3}}T^{\frac{2}{3}}+\Os\left(\sqrt T \log(T)\right).$$
\end{restatable}
Theorem \ref{thm:regradv} shows that \textsc{ETP} is capable of doing a sublinear regret in the adversarial case. With respect to the previous results for the stochastic case, this performance guarantee has two drawbacks. First, the bound scales roughly with $T^{2/3}$, which is worse than $\sqrt T$. Second, the guarantees only hold in expectation, not with high probability. This fact is a consequence of the use of $\varepsilon$: for unfortunate choices, the regret distribution could in principle have heavy tails. 
Lastly, note that the regret bound holds for any sequence $\{M_t\}_{t}$ that is oblivious, meaning that its values must not depend on the sampled $\varepsilon$. An assumption of this form is often called \textit{oblivious adversary}, and is standard in the literature. Generalizing to a sequence that can adapt to the agent's actions may be even more difficult.

\section{Related Work}

Online market making has been extensively studied in financial economics and market microstructure, with classical models such as Glosten--Milgrom capturing the interaction between market makers and traders with private information \citep{das2005learning,touzo2021information}. \citet{cesa2024market} recently introduced a regret-minimization framework for market making under fully censored (bandit) feedback, where the learner only observes whether a trade occurred. They show that under such feedback, sublinear regret is achievable only under restrictive smoothness assumptions on the distribution of trader valuations, and that regret lower bounds are significantly worse in more general settings. Notably, the extension of their framework to richer feedback models is explicitly identified as an important direction for future work. A summary of the best-known regret guarantees for online market making under different feedback models and assumptions on the price process is reported in Table~\ref{tab:comparison}.

The feedback structure studied in this paper concerns bandit problems with partial or censored observations. In most existing work, however, the censoring mechanism is exogenous and independent of the learner’s action. Action-dependent feedback has been studied in other online learning settings, such as feedback graphs, in which selecting an action reveals losses for neighboring actions. In contrast, our model exhibits a continuous and asymmetric form of action-dependent feedback: the informativeness of the observation depends on the chosen bid--ask spread, and informative feedback is obtained precisely in rounds where the instantaneous reward is zero. To the best of our knowledge, this specific feedback structure has not been previously analyzed.

From a bandit perspective, our problem is a continuous-action bandit with a two-dimensional action space. Classical continuum-armed bandits typically need Lipschitz or smoothness assumptions on the reward to achieve $\tilde \Os(\sqrt{T})$ regret; weaker assumptions give slower rates \citep{kleinberg2004nearly,bubeck2011x,liu2021smooth}. In contrast, we obtain $\tilde \Os(\sqrt{T})$ regret in stochastic and mean-reverting settings without smoothness assumptions, instead leveraging richer, action-dependent feedback.
Mean-reverting price dynamics are central in quantitative finance and are commonly modeled by autoregressive or Ornstein--Uhlenbeck processes \citep{vasicek1977equilibrium,bouchaud2003theory}. We adopt a learning-oriented perspective, imposing both local autoregressive dynamics and a weaker global mean-reverting drift condition based on cumulative deviations from the mean. This abstraction relaxes the classical i.i.d.\ assumption while remaining compatible with standard mean-reverting models, and yields high-probability regret guarantees in settings not covered by existing online market-making analyses.


\begin{table}[t]
\caption{
Comparison of regret guarantees for online market making under different feedback models.
The mixed feedback model captures order-book observability when no trade occurs.
}
\label{tab:comparison}
\centering
\begin{tblr}{
  colspec = {l c c c},
  row{1} = {font=\bfseries},
  row{2-Z} = {rowsep=3pt},
  hlines,
}
Price dynamics        & Stochastic  & Mean-Reverting & Adversarial\\
Bandit feedback \citep{cesa2024market} & $T^{\frac{2}{3}}$ & N/A & $T^{\frac{2}{3},*}$ \\
Mixed (\Cref{ass:order}, This paper) & $T^{\frac{1}{2}}$ & $T^{\frac{1}{2}}$ & $T^{\frac{2}{3}}$ \\
\end{tblr}

\vspace{0.5em}
{\footnotesize $^*$ assumes that the  \label{tab:sota} c.d.f. in  \Cref{ass:stoch_val} is Lipschitz continuous}
\end{table}





\section{Conclusions}
Motivated by a key feature of real limit order books, private valuations are revealed precisely when no trade occurs, we design a novel action-dependent feedback structure for Online Learning. Under this feedback, we prove substantially improved regret guarantees over standard bandit assumptions. In the stochastic independent setting, OPSR \Cref{alg:OPSR} attains regret $\sqrt{T\log(T)}$ (\Cref{thm:regstoc}). A minor modification that slows the update (\textsc{LazyOPSR}) yields the same bound in a more general mean-reverting price setting. Finally, (ETP) \Cref{alg:ETP}, combining an explore-then-commit scheme with the FTPL approach, achieves regret $T^{\frac{2}{3}}\log(T)^{\frac{1}{3}}$ in the adversarial case. 

\paragraph{Future works} While this work focuses on a stylized market-making model, we view it as a foundational step toward understanding how information structure affects learnability in online market-making. In particular, our analysis isolates the role of action-dependent feedback by deliberately abstracting away additional sources of complexity. A natural next direction is to incorporate inventory constraints, which play a central role in practical market making. Extending the proposed framework to settings in which the learner must control bid and ask quotes while maintaining a balanced inventory would allow one to study the interaction between risk management and action-dependent feedback, and to quantify the additional learning cost induced by inventory control.
Another key extension concerns the symmetry of the trading mechanism. The current model summarizes interaction with the market through a single private valuation and a reference price, enabling a clean analysis of feedback. A more realistic approach would model buy and sell pressure separately, matching bids and asks to distinct order streams and rewarding the market maker through the realized spread.

\acks{We thank a bunch of people and funding agency.}



\bibliography{yourbibfile}

\appendix

\section{\textsc{LazyOPSR}}\label{app:lazy_algo}

In this section, we present \textsc{LazyOPSR}, which is a modification of our main algorithm \textsc{OPSR} (see \Cref{alg:OPSR}) with a lazy update and a doubling-trick schedule. The rest of the structure follows exactly the same logic, so it does not need any particular comment.

\begin{algorithm2e}[t]
\everypar={\nl}
\RestyleAlgo{ruled}
\LinesNumbered
\caption{\textsc{LazyOPSR}}\label{alg:lazyOPSR}
\DontPrintSemicolon
\KwData{Time horizon $T$.}
Set $B_0\gets 0,A_0\gets 1$\;

$p\gets 0$

\For{$t = 1,\dots T$}{

    Play $(B_{t-1}, A_{t-1})$\;

    Update \Cref{eq:bid_ci} and \Cref{eq:mu_ci}\;

    \If{$t = 2^p$}{
        \For{$x \in \As_{t-1}$}{
    
            $\Theta_\text{bid}(x)\gets (\widehat \mu_t^\text{low}-x)\widehat F_t^\text{up}(x)$\;
        
            $\Theta_\text{ask}(x)\gets (x-\widehat \mu_t^\text{up})\widehat S_t^\text{up}(x)$\;
    
            $\Gamma_\text{bid}(x)\gets (\widehat \mu_t^\text{low}-x)\widehat F_t^\text{low}(x)$
    
            $\Gamma_\text{ask}(x)\gets (x-\widehat \mu_t^\text{up})\widehat S_t^\text{low}(x)$
        }
    
        $\Gamma_\text{bid}^* \gets \max_{x\in \As_{t-1}}\Gamma_\text{bid}(x)$\;
        $\Gamma_\text{ask}^* \gets \max_{x\in \As_{t-1}}\Gamma_\text{ask}(x)$\;
    
        $B_t\gets \min\{a\in \As_{t-1}: \Theta_\text{bid}(a)\ge \Gamma_\text{bid}^*\}$\;
    
        $A_t\gets \max\{a\in \As_{t-1}: \Theta_\text{ask}(a)\ge \Gamma_\text{ask}^*\}$\;

        $p \gets p+1$
    }
}
\end{algorithm2e}

\section{Proofs from \Cref{sec:stoc}}\label{app:stoc}

We start proving \Cref{eq:hoef2}.
\begin{proof}
    $M_t$ is an independent process bounded in $[0,1]$ by \Cref{ass:bounded}. Therefore, Theorem 9.2 by \cite{lattimore2020bandit}, for any $\varepsilon>0$ ensures that, for any $n$,
    $$\Prob\left(\exists t\le n,\quad\sum_{\tau=1}^t M_\tau-t\mu \ge \varepsilon\right)\le \exp\left(-\frac{2\varepsilon^2}{n}\right).$$
    Fixing $\delta>0$ and calling $\varepsilon=\sqrt{n\log(1/\delta)/2}$ gives
    $$\Prob\left(\exists t\le n,\quad\sum_{\tau=1}^t M_\tau-t\mu \ge \sqrt{\frac{n\log(1/\delta)}{2}}\right)\le \delta.$$
    Making a union bound for $n=2,4,8...2^{\lceil \log(T)\rceil}$ gives
    $$\Prob\left(\exists t\le T,\quad\sum_{\tau=1}^t M_\tau-t\mu \ge \sqrt{\frac{2^{\lceil \log(t)\rceil}\log(1/\delta)}{2}}\right)\le \delta \lceil \log(T)\rceil.$$
    At this point, we note that $2^{\lceil \log(t)\rceil}\le 2t$, so the former writes as
    $$\Prob\left(\exists t\le T,\quad\sum_{\tau=1}^t M_\tau-t\mu \ge \sqrt{t\log(1/\delta)}\right)\le \delta \lceil \log(T)\rceil,$$
    which, by definition of the sample mean, corresponds to
    $$\Prob\left(\exists t\le T,\quad \widehat\mu_t-\mu \ge \sqrt{\frac{\log(1/\delta)}{t}}\right)\le \delta \lceil \log(T)\rceil.$$
    Flipping the sign proves that also
    $$\Prob\left(\exists t\le T,\quad \mu-\widehat\mu_t \ge \sqrt{\frac{\log(1/\delta)}{t}}\right)\le \delta \lceil \log(T)\rceil.$$
    and then $\delta \gets \delta/(2\lceil \log(T)\rceil)$ ends the proof. 
\end{proof}

\ecdf*
\begin{proof}
    Let us fix $x\in \bigcap_{\tau=1}^t \As_\tau$. Then, by \Cref{eq:ecdf}
    \begin{align*}
        \widehat F_t(x)&= \frac{\sum_{\tau=1}^t \mathbbm{1}(V_\tau\le x)}{t}\\
        &= \frac{\sum_{\tau=1}^t \mathbbm{1}(V_\tau^\text{clip}\le x)}{t}\\
        &=\frac{1}{t}\sum_{\tau=1}^t\chi_\tau(x).
    \end{align*}
    The second equality for the follows because $x\in \bigcap_{\tau=1}^t \As_\tau$, so
    \begin{enumerate}
        \item If $V_\tau<B_\tau$ then $V_\tau^\text{clip}=B_\tau$, so $V_\tau^\text{clip}\le x$ always holds.
        \item If $B_\tau\le V_\tau\le A_\tau$ then $V_\tau^\text{clip}=V_\tau$.
        \item If $V_\tau> A_\tau$ then $V_\tau^\text{clip}>A_\tau$ and $V_\tau^\text{clip}\le x$ never holds.
    \end{enumerate}
    At this point, by Assumptions \ref{ass:bounded} and \ref{ass:stoch_val}, $\widehat F_t(x)$ is the sample mean of a sequence of i.i.d. random variables bounded in $[0,1]$. Thus, Hoeffding's inequality ensures that, with probability at most $1-\delta$
    $$|\widehat F_t(x)-F(x)|\le \sqrt{\frac{\log(2/\delta)}{2t}}.$$
    To obtain a uniform bound, we use a union bound. Indeed, $\bigcap_{\tau=1}^t \As_\tau \subset \As$, with $|\As|\le \sqrt T+1$. Therefore, at the same time for every $t, x\in \bigcap_{\tau=1}^t \As_\tau$, with probability at least $1-\delta$, $$|\widehat F_t(x)-F(x)|\le \sqrt{\frac{\log(2T(\sqrt T+1)/\delta)}{2t}}\le \sqrt{\frac{3\log(3T/\delta)}{4t}},$$
    which completes the proof.
\end{proof}

\begin{prop}\label{prop:stoch_reg}
    Under Assumptions \ref{ass:bounded}, \ref{ass:stoch_val} and \ref{ass:stoch_mark}, for any choice $B_t,A_t$ as a $\Fs_{t-1}$ measurable sequence, with probability at least $1-\delta$
    $$\left |R_T-\Rs_T\right|\le \sqrt{8T\log(4/\delta)}.$$
\end{prop}
\begin{proof}
    By definition of regret and Equations \eqref{eq:regdef} and \eqref{eq:obj2},
    \begin{align*}
        R_T&=\sup_{b,a\in [0,1]}\sum_{t=1}^T J_t(b,a)-J_t(B_t,A_t)\\
        &=T\left[\sup_{b,a\in [0,1]}(\widehat \mu_T-b)F(b)+(a-\widehat \mu_T)S(a)\right]-\sum_{t=1}^TJ_t(B_t,A_t).
    \end{align*}
    To bound the discrepancy with $\Rs_t$, we are going to compare the two parts with the corresponding parts of \Cref{eq:pseudo}. First, by Hoeffding's inequality, with probability at least $1-\delta$,

    \begin{align*}
        &\left|T\left[\sup_{b,a\in [0,1]}(\widehat \mu_T-b)F(b)+(a-\widehat \mu_T)S(a)\right]-T\cdot \sup_{b,a\in [0,1]}J(b,a)\right|\\
        &\qquad \qquad \le T\left|\sup_{b,a\in [0,1]}(\widehat \mu_T-\mu)F(b)+(\mu-\widehat \mu_T)S(a)\right|\\
        &\qquad \qquad \le 2T|\widehat \mu_T-\mu|\le 2T\sqrt{\frac{\log(2/\delta)}{2T}}=\sqrt{2T\log(2/\delta)}.
    \end{align*}
    Let us focus on the second one.
    \begin{align*}
        \sum_{t=1}^TJ_t(B_t,A_t)-\sum_{t=1}^T J(B_t,A_t) &= \sum_{t=1}^TF(B_t)(M_t-\mu)+S(A_t)(\mu-M_t)
    \end{align*}
    Thanks to \Cref{ass:stoch_mark}, $M_t$ is sampled at each time step independently from the past. Therefore, as $B_t,A_t$ form a $\Fs_{t-1}$ measurable sequence, the previous sum is a martingale, whose increments are bounded by
    $$\max_{b,a,m\in [0,1]}|F(b)(m-\mu)+S(a)(\mu-m)|\le 2.$$
    Therefore, another application of the Azuma-Hoeffding's inequality gives, with probability at least $1-\delta$
    $$\sum_{t=1}^TJ_t(B_t,A_t)-\sum_{t=1}^T J(B_t,A_t)\le \sqrt{2T\log(2/\delta)}.$$

    The union of the former two events completes the proof.
\end{proof}

\regretstoch*
\begin{proof}    
    As our algorithm works on the discrete set $\As$, it is necessary to reduce the supremum over the continuous interval to a maximum over $\As$. Let $\tilde b,\tilde a$ such that
    $$J(\tilde b, \tilde a) \ge \sup_{b,a\in [0,1]}J(b,a)-T^{-1}.$$
    The existence of such a pair follows by definition of supremum. Now, let
    $$b^\star:=\min\{x\in \As: x\ge \tilde b\},\qquad a^\star:=\max\{x\in \As: x\le \tilde a\}.$$
    By definition of $\As$, $b^\star-\tilde b\le T^{-1/2}$ and $\tilde a-a^\star\le T^{-1/2}$. Then,
    \begin{align*}
        (\mu-b^\star)F(b^\star)+(a^\star-   \mu)S(a^\star) &\ge ( \mu-b^\star)F(\tilde b)+(a^\star-   \mu)S(\tilde a)\\
        &\ge ( \mu-\tilde b)F(\tilde b)+(\tilde a-   \mu)S(\tilde a)-2T^{-1/2},
    \end{align*}
    where the first inequality comes from the fact that $F(x)$ is non-decreasing and $S(x)$ is non-increasing, and the second by the bound on the differences $b^\star-\tilde b, \tilde a-a^\star$. This entails that
    
    $$T\cdot \sup_{b,a \in \As}J(b, a)\ge T\cdot J(b^\star, a^\star)\ge T\cdot \sup_{b,a\in [0,1]}J(b,a) - 2\sqrt T-1.$$

    This fact allows us to focus on the regret w.r.t. the optimum within $\As$. As $\As$ is a finite set, there exists
    $$b^\star,a^\star\in \argmax_{b,a\in \As} J(b,a).$$
    We still name them in this way, with a small overload of notation. Under our assumptions, \Cref{lem:survival} ensures that $b^\star,a^\star$ are in $\As_t$ at any step $t$. 
    The following inequality holds by design of the algorithm
    \begin{align}
        (A_{t+1}-\widehat \mu_t^\text{up})\widehat S_t^\text{up}(A_{t+1})&=\Theta_\text{ask,t}(A_{t+1})\nonumber\\
        &\ge \Gamma_\text{ask,t}(a^\star)\nonumber\\
        &=(a^\star-\widehat \mu_t^\text{up})\widehat S_t^\text{low}(a^\star)\label{eq:shrinkask}.
    \end{align}
    Where the inequality comes from the fact that $a^\star\in \As_t$.
    In a symmetric way, one also has
    \begin{equation}
        (\widehat \mu_t^\text{low}-B_{t+1})\widehat F_t^\text{up}(B_{t+1})\ge (\widehat \mu_t^\text{low}-b^\star)\widehat F_t^\text{low}(b^\star)\label{eq:shrinkbid}
    \end{equation}

    At the same time, by definition of $E$, one also has
    $$(A_{t+1}-\widehat \mu_t^\text{up})\widehat S_t^\text{up}(A_{t+1}) - (A_{t+1}-\mu)S(A_{t+1})\le \psi(t,\delta)+\phi(t,\delta)$$
    $$(a^\star-\mu)S(a^\star)-(a^\star-\widehat \mu_t^\text{up})\widehat S_t^\text{low}(a^\star)\le \psi(t,\delta)+\phi(t,\delta),$$
    and analogous equation holding for the bids. 
    Replacing these results in equations \eqref{eq:shrinkbid} and \eqref{eq:shrinkask} gives
    \begin{equation}
        (\mu-b^\star)F(b^\star)-(\mu-B_{t+1})F(B_{t+1})\le 2\psi(t,\delta)+2\phi(t,\delta),\label{eq:boundbid}
    \end{equation}
    \begin{equation}
        (a^\star-\mu)S(a^\star)-(A_{t+1}-\mu)S(A_{t+1})\le 2\psi(t,\delta)+2\phi(t,\delta).\label{eq:boundask}
    \end{equation}
    Putting together all previous passages gives
    \begin{align*}
        \Rs_T&=T\cdot \sup_{b,a\in [0,1]}J(b,a)-\sum_{t=1}^T J(B_t,A_t)\\
        &\le \sum_{t=1}^T J(b^\star,a^\star)-J(B_t,A_t)+\sqrt T+1\\
        &\le \sqrt T+2 + \sum_{t=1}^{T-1} J(b^\star,a^\star)-(\mu-B_{t+1})F(B_{t+1})-(A_{t+1}-\mu)S(A_{t+1})\\
        &\le \sqrt T+2 + \sum_{t=1}^{T-1} 4\psi(t,\delta)+4\phi(t,\delta).
    \end{align*}
    The last sums can be upper-bounded in an explicit way:
    $$\sum_{t=1}^{T-1} \psi(t,\delta)=\sum_{t=1}^{T-1}\sqrt{\frac{3\log(3T/\delta)}{4t}}\le \sqrt{3T\log(3T/\delta)},$$
    and
    $$\sum_{t=1}^{T-1} \phi(t,\delta)=\sum_{t=1}^{T-1}\sqrt{\frac{\log(3\log(T)/\delta)}{t}}\le \sqrt{4T\log(3\log(T)/\delta)}.$$
    This completes the proof.
\end{proof}

\begin{lem}[Elimination scheme monotonicity]\label{lem:elim}
    For every $t=1,\dots T-1$ we have $\Gamma_\text{t,bid}^*\le \Gamma_\text{t+1,bid}^*$ and
    $\Gamma_\text{t,ask}^*\le \Gamma_\text{t+1,ask}^*$.
\end{lem}
\begin{proof}
    We prove $\Gamma_\text{t,bid}^*\le \Gamma_\text{t+1,bid}^*$, as the other part is analogous. By definition, 
    $$\Gamma_\text{t,bid}^*=\max_{x\in \As_{t-1}}(\widehat \mu_t^\text{low}-x)\widehat F_t^\text{low}(x).$$
    Call $x^*$ the arm realizing this maximum. As $x^*$ had the highest lower bound at $t$, it is still active at $t+1$,
    \begin{align*}\Gamma_\text{t+1,bid}^*&\ge (\widehat \mu_{t+1}^\text{low}-x^*)\widehat F_{t+1}^\text{low}(x^*)\\
        &\ge (\widehat \mu_{t}^\text{low}-x^*)\widehat F_{t}^\text{low}(x^*)=\Gamma_\text{t,bid}^*,
    \end{align*}
    where the second inequality comes from the monotonic structure of \cref{eq:bid_ci} and \cref{eq:mu_ci}.
\end{proof}

\begin{lem}\label{lem:survival}
    Under Assumptions \ref{ass:bounded}, \ref{ass:stoch_val}, \ref{ass:order}, and \ref{ass:stoch_mark}. Let
    $$b^\star,a^\star\in \argmax_{b,a\in \As} J(b,a).$$
    Under $E^c$, the failure event in \Cref{eq:failure}, the optimal arms $b^\star,a^\star\in \bigcap_{t=1}^T\As_t$ while running \Cref{alg:OPSR}.
\end{lem}
\begin{proof}
    We prove that for any time step $t$, $A_t\ge a^\star$. The proof for the optimal bid is equivalent to reversing the signs. This completes the statement: as $b^\star\le \mu\le a^\star$, proving that, at any time-step $t$, $A_t\ge a^\star$ and $b^\star\ge B_t$ implies that both are always in $\As_t$.

    By design of algorithm \ref{alg:OPSR}, a necessary condition to eliminate $a^\star$ is that at some $t$ there is $a< a^\star$ such that    $$\Theta_\text{ask}(a^\star) < \Gamma_\text{t,ask}=\Gamma_\text{ask}(a).$$   
    
    We first note that for this condition to occur, it is necessary that $a\ge \widehat \mu_t^\text{up}$. Indeed, if this condition is not satisfied, applying \cref{lem:elim} gives
    $$0=\Gamma_\text{0,ask}\le \Gamma_\text{t,ask}=\Gamma_\text{ask}(a)<0.$$
    We can therefore continue the proof assuming $a- \widehat \mu_t^\text{up}$ to be positive.
    
    The condition $\Theta_\text{ask}(a^\star) < \Gamma_\text{ask}(a)$ wites as
    $$(a^\star-\widehat \mu_t^\text{up})\widehat S_t^\text{up}(a^\star) < (a-\widehat \mu_t^\text{up})\widehat S_t^\text{low}(a).$$
    By definition of event $E$ and $\widehat S_t^\text{up}, \widehat S_t^\text{low}$ we have
    $\widehat S_t^\text{low}(a)\le S(a)\le \widehat S_t^\text{up}(a)$ and $\mu\le \mu_t^\text{up}$.
    It follows that
    \begin{align}
        (a-\widehat \mu_t^\text{up})\widehat S_t^\text{low}(a) &\le (a-\widehat \mu_t^\text{up}) S(a)\\
        &= (a-\mu) S(a)+(\mu-\widehat \mu_t^\text{up}) S(a)\\
        &\le (a^\star-\mu) S(a^\star)+(\mu-\widehat \mu_t^\text{up}) S(a) \label{eq:optim}\\
        &\le (a^\star-\mu) S(a^\star)+(\mu-\widehat \mu_t^\text{up}) S(a^\star)\label{eq:ge}\\
        & = (a^\star-\widehat \mu_t^\text{up})S(a^\star)\\
        & \le (a^\star-\widehat \mu_t^\text{up})\widehat S_t^\text{up}(a^\star).
    \end{align}
    Passage \eqref{eq:optim} comes from the optimality of $a^\star$, while \eqref{eq:ge} from the fact that $a<a^\star$, so $S(a^\star)<S(a)$ and $\mu-\widehat \mu_t^\text{up}$ is negative.
\end{proof}

\section{Proofs from \Cref{sec:gene}}

\begin{thm}\label{thm:newmart}
    Let $X_n$ for $n\in \mathbb N$ be a stochastic process bounded in $[-\sigma/2,\sigma/2]$ (for some $\sigma>0)$, and $S_n:=\sum_{i=1}^nX_i$ such that for any $n$
    \begin{equation}
        \E[X_{n+1}|\Fs_n]\cdot S_n\le 0.\label{eq:mra}
    \end{equation}
    If $n\sigma^2>2$, with probability at least $1-\delta$,
    $$|S_n|\le \sqrt{2n\sigma^2\log(2n/\delta)}.$$
\end{thm}
\begin{proof}
    For any $\lambda \in \mathbb R$, the following inequalities hold
    \begin{align*}
        \E[\exp(\lambda S_{i+1})|\Fs_i] &= \E[\exp(\lambda S_{i}+\lambda \E[X_{i+1}|\Fs_i]+\lambda (X_{i+1}-\E[X_{i+1}|\Fs_i]))|\Fs_i]\\
        &= \E[\exp(\lambda S_{i}+\lambda \E[X_{i+1}|\Fs_i])|\Fs_i]\\
        &\qquad \cdot \E[\exp(\lambda (X_{i+1}-\E[X_{i+1}|\Fs_i]))|\Fs_i]\\
        &\le \exp(\lambda S_{i})\exp(\lambda \E[X_{i+1}|\Fs_i]) e^{\frac{\lambda^2\sigma^2}{2}}.
    \end{align*}
    The last steps follow from the fact that 
    $X_{i+1}-\E[X_{i+1}|\Fs_i]$ is zero-mean, independent from the past (by definition of conditional expectation), and bounded in $[-\sigma,\sigma]$, so also $\sigma-$sub-Gaussian.

    By definition,
    $$\exp(\lambda S_{i})\exp(\lambda \E[X_{i+1}|\Fs_i])=(\mathbbm{1}\{S_i\le 0\}+\mathbbm{1}\{S_i> 0\})\exp(\lambda S_{i})\exp(\lambda \E[X_{i+1}|\Fs_i]).$$

    Using \Cref{eq:mra} it follows that only one between $S_i$ and $X_{i+1}$ can be positive.
    Therefore, the first term is
    $$\mathbbm{1}\{S_i\le 0\}\exp(\lambda S_{i})\exp(\lambda \E[X_{i+1}|\Fs_i])\le \mathbbm{1}\{S_i\le 0\}\exp(\lambda \E[X_{i+1}|\Fs_i]).$$
    On the other side,
    $$\mathbbm{1}\{S_i> 0\}\exp(\lambda S_{i})\exp(\lambda \E[X_{i+1}|\Fs_i])\le \mathbbm{1}\{S_i> 0\}\exp(\lambda S_{i}).$$
    Together, the two inequalities imply
    \begin{align*}
        \E[\exp(\lambda S_{i+1})|\Fs_i]&\le e^{\frac{\lambda^2\sigma^2}{2}}\max\{\exp(\lambda S_{i}), \exp(\lambda)\}\\
        &\le e^{\frac{\lambda^2\sigma^2}{2}}\left(\exp(\lambda S_{n})+\exp(\lambda)\right).
    \end{align*}
    Which, by induction, means
    $$\E[\exp(\lambda S_{n})]\le e^{\frac{n\lambda^2\sigma^2}{2}}+e^\lambda\sum_{m=0}^{n-1}e^{\frac{m\lambda^2\sigma^2}{2}}\le n e^{\frac{n\lambda^2\sigma^2}{2}}+ne^{\lambda}.$$
    
    Let us $\lambda = t/{n\sigma^2}$, so that
    $$e^{\frac{n\lambda^2\sigma^2-2\lambda t}{2}} \to e^{-\frac{t^2}{2n\sigma^2}}\qquad e^{\lambda(1-t)}\to e^{\frac{t-t^2}{n\sigma^2}}$$
    and Markov's inequality:    
    \begin{align*}
        \Prob(S_n>t)\le\exp(-\lambda t)\E[\exp(\lambda S_{n})]&\le e^{\frac{n\lambda^2\sigma^2-2\lambda t}{2}}+\sum_{m=0}^{n-1}e^{\frac{m\lambda^2\sigma^2}{2}}e^{\lambda(1-t)}\\
        &\le e^{\frac{n\lambda^2\sigma^2-2\lambda t}{2}}+n\max\{,e^{\lambda(1-t)}\}\\
        &\le ne^{-\frac{t^2}{2n\sigma^2}}+ne^{\frac{t-t^2}{n\sigma^2}}\\
        &\overset{t\ge 2}{\le} 2ne^{-\frac{t^2}{2n\sigma^2}}.
    \end{align*}
    Fixing $t=\sqrt{2n\sigma^2\log(2n/\delta)}$, the former passages show that
    \begin{align*}
        \Prob\left(S_n>\sqrt{2n\sigma^2\log(n/\delta)}\right)\le 2ne^{-\frac{2n\sigma^2\log(n/\delta)}{2n\sigma^2}}=\delta.\\
    \end{align*}
    Changing the sign $X_i\to -X_i$ in the whole proof shows the bound the other way round and completes the proof (note that \cref{eq:mra} does not change)
\end{proof}

\martingaleMR*
\begin{proof}
    We prove the two parts separately.
    
    \textbf{(Part 1)} By superposition of the effects, we can write
    $M_t=\overline M_t+\widetilde M_t$, where, calling $\epsilon_t:=\eta_t-\mu$,
    $$\widetilde M_{t+1}=\sum_{\tau=0}^{k-1}\gamma_\tau \widetilde M_{t-\tau}+(1-\gamma)\epsilon_{t+1}\qquad \overline M_{t+1}=\sum_{\tau=0}^{k-1}\gamma_\tau \overline M_{t-\tau}+(1-\gamma)\mu.$$

    By convention, we assume $\widetilde M_t=0$ for $t\le 0$, moving the initial conditions to the deterministic part. 
    
    The stochastic part satisfies the following equations
    \begin{align*}
        \widetilde S_n:=\sum_{t=1}^n \widetilde M_{t}&=\sum_{t=0}^{n-1}\sum_{\tau=0}^{k-1}\gamma_\tau \widetilde M_{t-\tau}+(1-\gamma)\epsilon_{t+1}\\
        &=(1-\gamma)\xi_n+\sum_{\tau=0}^{k-1}\sum_{t=0}^{n-1}\gamma_\tau \widetilde M_{t-\tau}\\
        &=(1-\gamma)\xi_n+\sum_{\tau=0}^{k-1}\gamma_\tau \widetilde S_{n-1-\tau},
    \end{align*}
    where    $$\xi_n:=\sum_{t=1}^n\epsilon_t.$$
    We can prove by induction that, for any $t\le n$, $\widetilde S_t$ is $\sqrt n/2-$subgaussian. For $S_1$ the result holds trivially as $\epsilon_1\in [0,1]$. For general $t+1\le n$, denoting $\|\cdot\|_{\psi_2}$ the Orclidz norm, we have
    \begin{align*}
        \|\widetilde S_{t+1}\|_{\psi_2}&\le (1-\gamma)\|\xi_{t+1}\|_{\psi_2}+\sum_{\tau=0}^{k-1}\gamma_\tau\|\widetilde S_{t-\tau}\|_{\psi_2}\\
        &\le (1-\gamma)(\sqrt {t+1}/2)+\gamma \max_{\tau \le t}\|\widetilde S_{\tau}\|_{\psi_2}\\
        &\le (1-\gamma)(\sqrt {t+1}/2)+\gamma (\sqrt {t+1}/2)=\sqrt {t+1}/2.
    \end{align*}
    As the Orclidz norm corresponds to the sug-Gaussian parameter, this shows that, at any time-step $t$,
    \begin{equation}
        \forall t\in \mathbb N\qquad \widetilde S_t\text{ is } \frac{\sqrt t}{2}\text{-sub-Gaussian.}\label{eq:stoch}
    \end{equation}

    A similar decomposition holds for the deterministic part
    \begin{align*}
        \overline S_n:=\sum_{t=1}^n \overline M_{t}&=\sum_{t=0}^{n-1}\sum_{\tau=0}^{k-1}\gamma_\tau \overline M_{t-\tau}+(1-\gamma)\mu\\
        &=(1-\gamma)n\mu+\sum_{\tau=0}^{k-1}\sum_{t=0}^{n-1}\gamma_\tau \overline M_{t-\tau}\\
        &=(1-\gamma)n\mu+\sum_{\tau=0}^{k-1}\gamma_\tau \left(\overline S_{n-1-\tau}+\sum_{q=1}^{\tau}M_{-q}\right).
    \end{align*}
    This time, we want to prove by induction that
    $$|\overline S_{t}-t\mu|\le \frac{\gamma (\mu + k)}{1-\gamma}=:C.$$
    Indeed,
    \begin{align*}
        |\overline S_{t+1}-(t+1)\mu|&\le \left|(1-\gamma)(t+1)\mu+\sum_{\tau=0}^{k-1}\gamma_\tau \left(\overline S_{t-\tau}+\sum_{q=1}^{\tau}M_{-q}\right)-(t+1)\mu\right|\\
        &=\left|(1-\gamma)(t+1)\mu+\sum_{\tau=0}^{k-1}\gamma_\tau Z_\tau-(t+1)\mu\right|,
    \end{align*}
    where
    $$Z_\tau=S_{t-\tau}+\sum_{q=1}^{\tau}M_{-q},\qquad Z_\tau\in [(t-\tau)\mu-C, (t-\tau)\mu+C+\tau],$$
    since, by induction, $|S_{t-\tau}-(t-\tau)\mu|\le C$ and $0\le \sum_{q=1}^{\tau}M_{-q}\le \tau$. For this reason,
    $$\gamma ((t-k)\mu-C)\le \sum_{\tau=0}^{k-1}\gamma_\tau Z_\tau\le \gamma(t\mu+k+C).$$

    Taking the two extrema, we have
    \begin{align*}
        (1-\gamma)(t+1)\mu+\sum_{\tau=0}^{k-1}\gamma_\tau Z_\tau-(t+1)\mu &\le (1-\gamma)(t+1)\mu+\gamma(t\mu+k+C)\\
        &\qquad -(t+1)\mu\\
        &= \gamma (k +C-\mu).
    \end{align*}
    and
    \begin{align*}
        (1-\gamma)(t+1)\mu+\sum_{\tau=0}^{k-1}\gamma_\tau Z_\tau-(t+1)\mu &\ge (1-\gamma)(t+1)\mu+\gamma ((t-k)\mu-C)\\
        &\qquad -(t+1)\mu\\
        &= -\gamma (\mu+k+C).
    \end{align*}
    Replacing
    $C=\frac{\gamma (\mu + k)}{1-\gamma}$, gives the same quantity. As $\mu\le 1$, this proves that
    \begin{equation}
        \forall t\in \mathbb N \qquad |\overline S_t-t\mu|\le \frac{\gamma (k+1)}{1-\gamma}.\label{eq:det}
    \end{equation}

    It is now possible to recollect the previous formulas to give an upper bound to the discrepancy between $\mu$ and $\widehat \mu_t$. From \cref{eq:stoch}, it follows that, w.p. at least $1-\delta$,
    $$\forall 1\le t\le T\qquad \left|\frac{\widetilde S}{t}\right|\le \sqrt{\frac{\log(2T/\delta)}{4t}}.$$
    
    From \cref{eq:det}, we get that, with the same probability
    \begin{align*}
        |\mu-\widehat\mu_t| =\left|\mu-\frac{\sum_{\tau=1}^tM_\tau}{t}\right|&=\left|\mu-\frac{\widetilde S_t+\overline S_t}{t}\right|\\
        &\le \frac{\gamma (k+1)}{(1-\gamma)t}+\left|\frac{\widetilde S_t}{t}\right|\le \frac{\gamma (k+1)}{(1-\gamma)t}+\sqrt{\frac{\log(2T/\delta)}{4t}},
    \end{align*}
    which completes the proof.

    \textbf{(Part 2)}
    The former \cref{thm:newmart} applies with $\sigma=1$. At any time-step $t>2$, with probability at least $1-\delta$,
    \begin{align*}
        |\widehat \mu_t - \mu| &=\left|\frac{1}{t}\sum_{\tau=1}^t(M_t - \mu)\right|\\
        &=\left|\frac{S_t}{t}\right|\le \sqrt{\frac{2\log(2t/\delta)}{t}}.
    \end{align*}
    For $t\le 2$, we note that the inequality is immediately verified by the fact that $S_t/t\in [-1,1]$ a.s..
    Making a union bound over $t=1,...T$ ends the proof.
\end{proof}

\regretmr*
\begin{proof}
    The proof is done under $E^c$, for $E$ defined as in \cref{eq:failure} in $\phi\gets \overline \phi$ and $\mu$ coming from \cref{ass:mr_mar}. Thanks to assumptions \ref{ass:bounded} \ref{ass:stoch_val} and \ref{ass:mr_mar},  \cref{prop:ecdf} and \cref{thm:martingaleMR} show that $\Prob(E)<2\delta$.

    Let us call $p=0,\dots \lfloor \log T\rfloor$ the current phase of the algorithm. By design, the sequences of bids and asks can be written as sequences $B_{(p)},A_{(p)}$ each repeated $2^{p}$ times.

    By equation \eqref{eq:regdef}, the regret writes as follows
    \begin{align*}
        R_T &\le \sqrt T+\sup_{b,a\in \As}\sum_{t=1}^T J_t(b,a)-J_t(B_t,A_t)\\
        &\le \sqrt T+\sum_{p=0}^{\lfloor \log T\rfloor}\sup_{b,a\in \As}\sum_{\tau=0}^{2^{p}-1} J_{2^p+\tau}(b,a)-J_{2^p+\tau}(B_{2^p+\tau},A_{2^p+\tau})\\
        &\overset{(*)}{\le} \sqrt T+\sum_{p=0}^{\lfloor \log T\rfloor}\sup_{b,a\in \As}\sum_{\tau=0}^{2^{p}-1} J_{2^p+\tau}(b,a)-J_{2^p+\tau}(B_{(p)},A_{(p)})\\
        &=:\sqrt T+\sum_{p=0}^{\lfloor \log T\rfloor}R_{T,\ell}.
    \end{align*}
    Where $(*)$ follows from the structure of \textsc{LazyOPSR}, which only switches arms when $p$ changes. 
    When this happens, that is, for $t=2^p$, we can apply  Equations \eqref{eq:boundbid} and \eqref{eq:boundask}:
    \begin{equation}
    \sup_{b,a\in \As}J(b,a)-J(B_{(p)},A_{(p)})\le 4\psi(2^p,\delta)+4\overline\phi(2^p,\delta)=:\kappa_1(p,\delta)
    \label{eq:lampo_blu}
    \end{equation}
    Coming back to the regret, the following upper bound holds for any $p$:
    \begin{align*}
        R_{T,p}&\le \sup_{b,a\in \As}\sum_{\tau=0}^{2^{p}-1} F(b)(M_{2^p+\tau}-b)+S(a)(a-M_{2^p+\tau})\\
        &\qquad -\sum_{\tau=0}^{2^{p}-1}F(B_{(p)})(M_{2^p+\tau}-B_{(p)})+S(A_{(p)})(A_{(p)}-M_{2^p+\tau})\\
        &= \sup_{b,a\in \As} F(b)\left( \sum_{\tau=0}^{2^{p}-1}M_{2^p+\tau}-2^pb\right)+S(a)\left(2^p a-\sum_{\tau=0}^{2^{p}-1} M_{2^p+\tau}\right)\\
        &\qquad -F(B_{(p)})\left(\sum_{\tau=0}^{2^{p}-1} M_{2^p+\tau}-2^p B_{(p)}\right)+S(A_{(p)})\left(2^p A_{(p)}-\sum_{\tau=0}^{2^{p}-1} M_{2^p+\tau}\right)\\
        &= 2^p\left[\sup_{b,a\in \As} F(b)\left( \widehat \mu_{(p)}-b\right)+S(a)\left( a-\widehat \mu_{(p)}\right)\right]\\
        &\qquad -2^p\left[F(B_{(p)})\left(\widehat \mu_{(p)}- B_{(p)}\right)+S(A_{(p)})\left(A_{(p)}-\widehat \mu_{(p)}\right)\right]\qquad \widehat \mu_{(p)}:=2^{-p}\sum_{\tau=0}^{2^{p}-1} M_{2^p+\tau}.
    \end{align*}
    In the previous result, $\widehat \mu_{(p)}$ corresponds to the sample mean of the prices $M_t$ during phase $p$. Crucially, when compared to the global sample mean at the beginning of the phase, $\widehat \mu_{2^p}$ and the one at the end, $\widehat \mu_{2^{p+1}}$, the following equation holds true.
    $$\widehat \mu_{2^{p+1}}=\frac{\widehat \mu_{2^p} + \widehat \mu_{(p)}}{2}\implies \widehat \mu_{(p)}=2\widehat \mu_{2^{p+1}}-\widehat \mu_{2^p}.$$
    By definition of event $E^c$, both $\widehat \mu_{2^p}$ and $\widehat \mu_{2^{p+1}}$ are at most $\overline\phi(t,\delta)$-far from the true mean, with $t=2^p$ / $2^{p+1}$ respectively. Therefore, 
    \begin{equation}
        |\mu_{(p)}-\mu|\le 2\overline\phi(2^{p+1},\delta)+\overline\phi(2^{p},\delta):=\kappa_2(p,\delta)
        \label{eq:lampo_rosso}
    \end{equation}
    We can use Equations \eqref{eq:lampo_blu} and \eqref{eq:lampo_rosso} to complete the regret bound on $R_{T,p}$. In fact
    \begin{align*}
        R_{T,p}&\le 2^p\left[\sup_{b,a\in \As} F(b)\left( \widehat \mu_{(p)}-b\right)+S(a)\left( a-\widehat \mu_{(p)}\right)\right]\\
        &\qquad -2^p\left[F(B_{(p)})\left(\widehat \mu_{(p)}- B_{(p)}\right)+S(A_{(p)})\left(A_{(p)}-\widehat \mu_{(p)}\right)\right]\\
        &\overset{\cref{eq:lampo_rosso}}{\le}2^p\left[\sup_{b,a\in \As} F(b)\left( \mu-b\right)+S(a)\left( a-\mu\right)+2\kappa_2(p,\delta)\right]\\
        &\qquad -2^p\left[F(B_{(p)})\left(\mu- B_{(p)}\right)+S(A_{(p)})\left(A_{(p)}-\mu\right)-2\kappa_2(p,\delta)\right]\\
        &=2^{p+2}\kappa_2(p,\delta)+2^p\left [\sup_{b,a\in \As}J(b,a)-J(B_{(p)},A_{(p)})\right]\\
        &\overset{\cref{eq:lampo_blu}}\le 2^{p+2}\kappa_2(p,\delta)+2^p \kappa_1(p,\delta).
    \end{align*}
    The whole term $R_{T,p}$ is thus bounded by 
    
    $$2^{p+2}\kappa_2(p,\delta)+2^p \kappa_1(p,\delta)=2^{p+2}\psi(2^p,\delta)+2^{p+3}\overline\phi(2^p,\delta)+2^{p+3}\overline\phi(2^{p+1},\delta).$$

    From this point on, the proof depends on the exact definition of $\psi$ and $\overline\phi$. In both cases of \cref{ass:mr_mar} and \cref{ass:ou_mar}, the former writes as \cref{eq:bid_ci}. Therefore,
    $$2^{p+2}\psi(2^p,\delta)=\sqrt{\frac{3\log(3T/\delta)}{2^{p+2}}}=\sqrt{12\cdot 2^p\log(3T/\delta)}.$$
    The latter instead corresponds, by \cref{thm:martingaleMR}, to
    $$2^{p+3}\overline\phi(2^p,\delta)=\begin{cases}
        \sqrt{16\cdot 2^p\log(2T/\delta)}+\frac{8\gamma (k+1)}{(1-\gamma)}&\text{\Cref{ass:ou_mar}}\\
        \sqrt{256\cdot 2^p\log(2T/\delta)}&\text{\Cref{ass:mr_mar}},
    \end{cases}$$
    and
    $$2^{p+3}\overline\phi(2^{p+1},\delta)=\begin{cases}
        \sqrt{16\cdot 2^{p+1}\log(2T/\delta)}+\frac{8\gamma (k+1)}{(1-\gamma)}&\text{\Cref{ass:ou_mar}}\\
        \sqrt{256\cdot 2^{p+1}\log(2T/\delta)}&\text{\Cref{ass:mr_mar}}.
    \end{cases}$$    
    
    Replacing this values in the total regret achieves, for $$C_1=\sqrt{12\log(3T/\delta)}+\sqrt{16\log(2T/\delta)}+\sqrt{32\log(2T/\delta)},\qquad  C_2=\frac{8\gamma (k+1)}{(1-\gamma)}$$
    in \cref{ass:ou_mar} and $$C_1=\sqrt{12\log(3T/\delta)}+\sqrt{256\log(2T/\delta)}+\sqrt{518\log(2T/\delta)},\qquad C_2=0,$$
    the following expression
    \begin{align*}
        R_T&\le \sqrt T+\sum_{p=0}^{\lfloor \log T\rfloor}R_{T,\ell}\\
        &\le \sqrt T+\sum_{p=0}^{\lfloor \log T\rfloor}2^{p+2}\psi(2^p,\delta)+2^{p+3}\overline\phi(2^p,\delta)+2^{p+3}\overline\phi(2^{p+1},\delta)\\
        &\le \sqrt T+\sum_{p=0}^{\lfloor \log T\rfloor}2^{p/2}C_1+C_2\le \sqrt T+\frac{C_1\sqrt{2T} }{\sqrt 2-1}+C_2\log (T).
    \end{align*}  
    Replacing the constants ends the proof.
\end{proof}

\section{Proofs from \Cref{sec:adv}}

\regretadv*
\begin{proof}
    By definition and \Cref{eq:obj}, the regret can be written as follows, for
    $J_t(a,b):=F(b)(M_t-b)+S(a)(a-M_t)$
    \begin{align*}
        R_T&=\sup_{b,a\in [0,1]}\sum_{t=1}^T J_t(b,a)-J_t(B_{t-1},A_{t-1})\\
        &=\sqrt T+\underbrace{\sup_{b\in \As}\sum_{t=1}^T F(b)(M_t-b)-F(B_{t-1})(M_t-B_{t-1})}_{R_T^{\text{bid}}}\\
        &\qquad +\underbrace{\sup_{a\in \As}\sum_{t=1}^T S(a)(a-M_t)-S(A_{t-1})(A_{t-1}-M_t)}_{R_T^{\text{ask}}}.
    \end{align*}
    Below, we show how to bound the $R_T^{\text{bid}}$ part, relative to bidding. The result for $R_T^{\text{ask}}$ follows analogously by just flipping the signs.
    As \Cref{alg:ETP} starts choosing $B_t=0,A_t=1$ in the first $\kappa$ rounds, \Cref{prop:ecdf} ensures that for $\As_t=\As$, with probability at least $1-\delta$
    $$\forall x\in \As,\qquad |\widehat F(x)-F(x)|\le \frac{\sqrt{3\log(3T/\delta)}}{2\sqrt \kappa}.$$

    Fixing $\delta=T^{-1}$,
    since the bids belong to $[0,1]$, 

    $$\E \left[R_T^{\text{bid}}\right]\le \E \bigg[\underbrace{\sup_{b\in \As}\sum_{t=1}^T \widehat F(b)(M_t-b)-\widehat F(B_{t-1})(M_t-B_{t-1})}_{=:\widetilde R_T^{\text{bid}}}\bigg]+\frac{\sqrt{6\log(3T)}T}{\sqrt \kappa}+1.$$

    We now need to bound $\widetilde R_T^{\text{bid}}$.

    \begin{align*}
        \widetilde R_T^{\text{bid}} &=  \sup_{b\in \As}\sum_{t=1}^T\widehat F(b)(M_t-b)-\sum_{t=1}^T\widehat F(B_{t-1})(M_t-B_{t-1})\\
        &\le \kappa+  \underbrace{\sup_{b\in \As} \sum_{t=\kappa+1}^T\widehat F(b)(M_t-b)-\sum_{t=\kappa+1}^T\widehat F(B_{t})(M_t-B_{t})}_{\text{P2}}\\
        &\qquad +  \underbrace{\sum_{t=\kappa+1}^T\widehat F(B_{t})(M_t-B_{t})-\sum_{t=\kappa+1}^T\widehat F(B_{t-1})(M_t-B_{t-1})}_{\text{P2}}.
    \end{align*}
    In the former equations, we have split between the first $\kappa$ rounds, which are devoted to pure exploration, and the following $T-\kappa$, where the interesting parts of the algorithm take place. After this, the regret is split into two terms, $T1$ taking into account the difference between the clairvoyant and $B_{t}$, and $T2$ measuring the difference between $B_{t}$ and $B_{t-1}$. This decomposition, is relatively standard in the analysis of follow-the-leader algorithms.

    Let us examine the first part. As $0\le \varepsilon\le 1/\sqrt T$ almost surely, one has
    $$\left|\text{P1}-\sup_{b\in \As} \sum_{t=\kappa+1}^T\widehat F(b)(M_t+\varepsilon-b)+\sum_{t=\kappa+1}^T\widehat F(B_{t})(M_t+\varepsilon-B_{t})\right|\le 2\sqrt T.$$

    We are going to prove by induction that for any $\kappa\le K\le T$,
    $$\sup_{b\in \As} \sum_{t=\kappa+1}^T\widehat F(b)(M_t+\varepsilon-b)-\sum_{t=\kappa+1}^T\widehat F(B_{t})(M_t+\varepsilon-B_{t})\le 0.$$

    Now, we perform the inductive step, assuming the thesis holds for $K$
    \begin{align*}
        \sup_{b\in \As} \sum_{t=\kappa+1}^{K+1}\widehat F(b)(M_t+\varepsilon-b)&=\sum_{t=\kappa+1}^{K+1}\widehat F(B_{K+1})(M_t+\varepsilon-B_{K+1})\\
        &=\sum_{t=\kappa+1}^{K}\widehat F(B_{K+1})(M_t+\varepsilon-B_{K+1})+\widehat F(B_{K+1})(M_{K+1}+\varepsilon-B_{K+1})\\
        &\le \sup_{b\in \As} \sum_{t=\kappa+1}^{K}\widehat F(b)(M_t+\varepsilon-b)+\widehat F(B_{K+1})(M_{K+1}+\varepsilon-B_{K+1})\\
        &\le \sum_{t=\kappa+1}^{K}\widehat F(B_{t})(M_t+\varepsilon-B_{t})+\widehat F(B_{K+1})(M_{K+1}+\varepsilon-B_{K+1}).
    \end{align*}
    which completes the inductive part. This proves that
    \begin{equation}
        \text{P1}\le 2\sqrt T\label{eq:betheleader}.
    \end{equation}
    Then, we work on $\text{P2}$. This part is a random variable, and we start proving the following expectation upper bound. Let $f:[0,1]\to [0,1]$ be any fixed function and $\Delta_t:=\widehat \mu_{t+1}-\widehat \mu_t$. Since $\varepsilon$ has uniform law on $(0,1/\sqrt{T})$,
    \begin{align*}
        \E[f(B_{t})]&=\sqrt T\int_{0}^{T^{-1/2}}f(\text{argmax}_{b} \widehat F(b)(\widehat \mu_{t+1}+x-b))\ dx\\
        &=\sqrt T\int_{0}^{T^{-1/2}}f(\text{argmax}_{b} \widehat F(b)(\widehat \mu_{t}+\Delta_{t+1}+x-b))\ dx\\
        &=\sqrt T\int_{\Delta_{t+1}}^{T^{-1/2}+\Delta_{t+1}}f(\text{argmax}_{b} \widehat F(b)(\widehat \mu_{t}+x-b))\ dx\\
        &=\underbrace{\sqrt T\int_{0}^{T^{-1/2}}f(\text{argmax}_{b} \widehat F(b)(\widehat \mu_{t}+x-b))\ dx}_{\E[f(b_{t})]}\\
        &\qquad +\sqrt T\int_{T^{-1/2}}^{T^{-1/2}+\Delta_{t+1}}f(\text{argmax}_{b} \widehat F(b)(\widehat \mu_{t}+x-b))\ dx\\
        &\qquad -\sqrt T\int_0^{\Delta_{t+1}}f(\text{argmax}_{b} \widehat F(b)(\widehat \mu_{t}+x-b))\ dx\\
        &\le \E[f(B_{t})]+\sqrt T\Delta_{t+1}\le \E[f(B_{t})]+\sqrt{\frac{T}{(t+1)^2}},
    \end{align*}
    where the key step is the fact that the sample mean change $\Delta_t$ does not exceed $1/(t+1)$ from step $t$ to $t+1$, a consequence of \cref{ass:bounded}. If one is interested in the sum over $\sum_{t=\kappa+1}^T$, we get, for any sequence of functions $f_t$ that are independent on the noise,

    $$\E\left[\sum_{t=\kappa+1}^T f_t(B_{t})-f_t(B_{t-1})\right] \le \sum_{t=\kappa+1}^T \frac{\sqrt T}{t-\kappa}\le \sqrt{T}\log(T).$$

    Taking $f_t(b):=\widehat F(b)(M_t-b),$ that is bounded in $[0,1]$, this shows that
    $$\E\left[\sum_{t=\kappa+1}^T\widehat F(B_{t})(M_t-B_{t})-\sum_{t=\kappa+1}^T\widehat F(B_{t-1})(M_t-B_{t-1})\right] \le \sqrt T\log(T),$$

    i.e. $\E[\text{P2}]\le \sqrt T\log(T).$ Toghether with \Cref{eq:betheleader}, this proves that
    $$\E[\widetilde R_T^{\text{bid}}]\le \kappa+2\sqrt T+\sqrt T\log(T).$$

    As anticipated, the same passages show that
    $\E[\widetilde R_T^{\text{ask}}]\le \kappa+2\sqrt T+\sqrt T\log(T).$ Putting everything together,
    
    \begin{align*}
        \E[R_T]&\le\sqrt T+\E \left[R_T^{\text{bid}}\right]+\E \left[R_T^{\text{ask}}\right]\\
        &\le\sqrt T+\frac{\sqrt{6\log(3T)}T}{\sqrt \kappa}+2+\E \left[\widetilde R_T^{\text{bid}}\right]+\E \left[\widetilde R_T^{\text{ask}}\right]\\
        &\le\sqrt T+\frac{\sqrt{6\log(3T)}T}{\sqrt \kappa}+2+2\kappa+4\sqrt T+2\sqrt T\log(T).
    \end{align*}

    Replacing the value of $\kappa=\left \lceil \log(3T)^{\frac{1}{3}}T^{\frac{2}{3}}\right \rceil$, one gets,
    $$R_T\le 5\log(3T)^{\frac{1}{3}}T^{\frac{2}{3}}+\Os\left(\sqrt T \log(T)\right).$$
\end{proof}

\end{document}